\documentclass[runningheads]{llncs}
\usepackage[T1]{fontenc}
\usepackage{graphicx}
\usepackage{subfigure}
\usepackage{amsmath}
\usepackage{bm}

\usepackage{multicol}
\usepackage{multirow}

\usepackage{mathrsfs}
\usepackage{amsfonts}
\usepackage{tablefootnote}
\usepackage{booktabs}
\usepackage{color}
\usepackage[misc]{ifsym}
\usepackage[table]{xcolor}

\usepackage{makecell}
\usepackage[colorlinks,
linkcolor=blue,
anchorcolor=blue,
citecolor=blue,
urlcolor=blue]{hyperref}

\begin{document}
\title{Rethinking Autoencoders for Medical Anomaly Detection from A Theoretical Perspective}
%
%
\author{Yu Cai  \and  
Hao Chen\inst{(\textrm{\Letter})}  \and  
Kwang-Ting Cheng  
}
\authorrunning{Y. Cai et al.}
%
\institute{The Hong Kong University of Science and Technology, Hong Kong, China \\ 
\email{jhc@cse.ust.hk}}
\maketitle              

\begin{abstract}
Medical anomaly detection aims to identify abnormal findings using only normal training data, playing a crucial role in health screening and recognizing rare diseases. Reconstruction-based methods, particularly those utilizing autoencoders (AEs), are dominant in this field. They work under the assumption that AEs trained on only normal data cannot reconstruct unseen abnormal regions well, thereby enabling the anomaly detection based on reconstruction errors. However, this assumption does not always hold due to the mismatch between the reconstruction training objective and the anomaly detection task objective, rendering these methods theoretically unsound. This study focuses on providing a theoretical foundation for AE-based reconstruction methods in anomaly detection. By leveraging information theory, we elucidate the principles of these methods and reveal that the key to improving AE in anomaly detection lies in minimizing the information entropy of latent vectors. Experiments on four datasets with two image modalities validate the effectiveness of our theory. To the best of our knowledge, this is the first effort to theoretically clarify the principles and design philosophy of AE for anomaly detection. The code is available at \url{https://github.com/caiyu6666/AE4AD}. 

\end{abstract}
\section{Introduction}

In various medical scenarios, including health screening and rare disease recognition, collecting an adequate number of annotated abnormal images presents a challenge due to the diverse range of diseases and the scarcity of abnormal cases. Consequently, anomaly detection, which only necessitates normal samples during training, has emerged as a promising solution for such scenarios.

\begin{figure}
\centering
\includegraphics[width=0.7\linewidth]{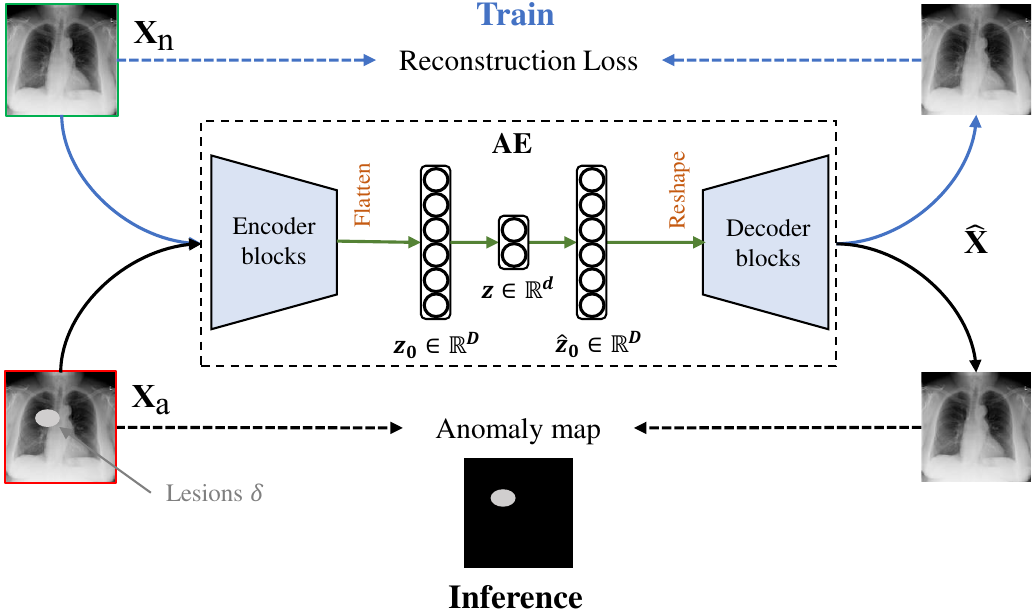}
\caption{Overview of the reconstruction AE-based AD. The model is trained to minimize reconstruction loss on normal images $\mathbf{X}_n$. During inference, lesions $\delta$ in abnormal images $\mathbf{X}_a$ are assumed unable to be reconstructed by the trained model.}
\label{fig:overview_rec}
\end{figure}

Anomaly detection (AD) is a fundamental machine learning problem \cite{chandola2009anomaly,pang2021deep} that aims to detect abnormal samples that deviate from expected normal patterns. It is typically formulated as one-class classification (OCC), where models are exclusively trained on normal data. The trained models should effectively detect samples that deviate from normal patterns, eliminating the need for abnormal training data.  
In the field of medical AD, reconstruction-based methods utilizing AEs have gained popularity as effective strategies \cite{baur2021autoencoders,schlegl2017unsupervised,schlegl2019f,zimmerer2018context,chen2018unsupervised,pinaya2022unsupervised}. Fig.~\ref{fig:overview_rec} depicts an overview of AE in AD. The AE is trained to reconstruct normal images, and during inference, the reconstruction error serves as an anomaly score, indicating the extent of deviation from normality. These methods operate on the \textit{assumption} that (i) a model trained solely on normal images can effectively reconstruct normal regions, but (ii) it cannot reconstruct unseen abnormal regions, resulting in high reconstruction errors in abnormal regions.

Considerable efforts have been made for reconstruction-based AD. To enhance the reconstruction quality of normal images (\textit{assumption} (i)), several methods \cite{baur2019deep,akcay2019ganomaly} introduced adversarial training to AE-based models. Mao et al. \cite{mao2020abnormality} proposed an automatic estimation of pixel-level uncertainty in reconstruction to mitigate unexpected reconstruction errors around normal boundaries. Bercea et al. \cite{bercea2023aes} refined the pseudo-healthy reconstruction using estimated dense deformation fields \cite{balakrishnan2019voxelmorph}. On the other hand, some approaches focus on preventing the reconstruction of abnormal regions (\textit{assumption} (ii)). Gong et al. \cite{gong2019memorizing} proposed to record prototypes of normal training patterns, which are used as inputs to the decoder to facilitate anomaly-free reconstruction. Zimmerer et al. \cite{zimmerer2018context} introduced an inpainting task to aid the model in repairing abnormal regions. You et al. \cite{you2022unified} proposed a transformer architecture with neighbor masked attention to prevent abnormal information from being reconstructed.

However, we notice that \textit{assumption} (ii) does not always hold due to the absence of relevant constraints in the training objectives. While existing methods incorporate various intuitive designs to prevent the reconstruction of anomaly and facilitate \textit{assumption} (ii), they rely on heuristics and lack a solid theoretical foundation, resulting in ambiguous optimal solutions. This lack of a theoretical foundation has led to erroneous arguments, as seen \cite{you2022unified,lu2024hierarchical}, that claim the existence of an "identical shortcut" in reconstruction networks, further contributing to confusion.
Bercea et al. \cite{bercea2023aes} attempted to tackle this problem by summarizing two desired properties for AE: sufficiency and minimality, which guarantee \textit{assumption} (i) and (ii) respectively. However, their analysis remains qualitative. 

This study aims to address this issue by providing a theoretical foundation to explain the workings of reconstruction-based methods and uncover their theoretically optimal solutions. Specifically, we first prove that proper latent dimensions in AE can effectively avoid the "identical shortcut" mentioned in previous works. Additionally, using information theory, we theoretically clarify the principle of AE in AD and establish its optimal solution. Our theory reveals that, apart from the reconstruction loss, another restriction should be imposed on AE to constrain the entropy of its latent space. Ideally, the entropy of the latent space should match that of normal data, enabling effective reconstruction of normal data while preventing the reconstruction of lesions. In experiments, we observe that simple latent dimension reduction is a powerful way to constrain the entropy of the latent space. The AD performance consistently aligns with our theory across varying latent dimensions, with extremely small latent dimensions (e.g., 4) often yielding optimal performance. Our contribution is summarized as follows:
\begin{itemize}
    \item We prove that with an appropriate latent dimension, AE for AD does not encounter the “identical shortcut” problem.
    \item Based on information theory, we provide an explanation of how AE operates in AD, and uncover its theoretically optimal solution.
    \item Experiments on four datasets with two image modalities present consistent trends that align with our proposed theory, validating its effectiveness.
\end{itemize}

\section{The pipeline and limitation of AE in anomaly detection}  \label{sec:preliminary}
As shown in Fig.~\ref{fig:overview_rec}, AE $\phi$ comprises an encoder $f_e$ and a decoder $f_d$. The encoder compresses the input image $\mathbf{X} \in \mathbb{R}^{C \times H \times W}$ into a compact latent vector $\mathbf{Z} = f_e(\mathbf{X}) \in \mathbb{R}^d$, and the decoder maps the latent vector back to the image space $\hat{\mathbf{X}} = f_d(\mathbf{Z}) \in \mathbb{R}^{C \times H \times W}$. We formally denote the normal image as $\mathbf{X}_n$ and the abnormal image as $\mathbf{X}_a$. In the case of typical medical images like chest X-rays and brain MRIs, each abnormal image $\mathbf{X}_a$ can be understood as the corresponding healthy version $\mathbf{X}_n$ with the addition of lesion regions $\delta$, i.e., $\mathbf{X}_a = \mathbf{X}_n + \delta$ \cite{chen2020unsupervised}. The training objective is to minimize the reconstruction loss on normal images:
\begin{equation}  \label{eq:train}
    \min \mathbb{E}[\Vert \phi(\mathbf{X}_n) - \mathbf{X}_n \Vert^2],
\end{equation}
then the well-trained AE is  expected to achieve the following ideal objectives:
\begin{align} \label{eq:obj1}
  & \hat{\mathbf{X}}_n = \phi(\mathbf{X}_n) \xrightarrow{} \mathbf{X}_n          \\ 
  & \hat{\mathbf{X}}_a = \phi(\mathbf{X}_a) = \phi(\mathbf{X}_n+\delta) \xrightarrow{} \mathbf{X}_n,       \label{eq:obj2}
\end{align} 
where $\hat{\mathbf{X}}_n$ and $\hat{\mathbf{X}}_a$ represent the reconstructions generated by AE. Consequently, the reconstruction error, $\mathcal{A}_{rec}(\mathbf{X}) = \Vert \hat{\mathbf{X}} - \mathbf{X} \Vert^2$, will be $\mathbf{0}$ if $\mathbf{X}=\mathbf{X}_n$ and $\delta^2$ if $\mathbf{X}=\mathbf{X}_a$, thereby highlighting the abnormal regions.

However, the performance of AE is suboptimal due to a mismatch between the training objective (Eq.~\ref{eq:train}) and the ideal task objectives (Eq.~\ref{eq:obj1} and \ref{eq:obj2}). While the training objective encourages AE to generate reconstructions that are identical to the network input, the expected reconstructions of abnormal images differ from the input during inference. This discrepancy can lead to false negatives, as the AE may successfully reconstruct some abnormal regions due to the generalization ability of deep neural networks. 
In the extreme case, if AE learns a function $\phi(\mathbf{X})=\mathbf{X}, \forall \mathbf{X}$, it perfectly satisfies the training objective, but cannot detect any anomalies. This phenomenon is called "identical shortcut" \cite{you2022unified,lu2024hierarchical}.
To better understand this issue and seek potential solutions, we conduct a thorough theoretical analysis of properties of AE in the following.

\section{Theoretical analysis of AE in anomaly detection}  \label{sec:theory}
This section theoretically analyzes AE in AD. We firstly elucidate its fundamental property of reconstruction to establish that an AE with an appropriate latent dimension effectively avoids the "identical shortcuts", and then leverage information theory to derive the optimal solution for AE.

\subsection{An inherent property of AE}
While previous papers \cite{you2022unified,lu2024hierarchical} attribute the failure of \textit{assumption} (ii) to identical shortcut, we contend that such a statement is erroneous as it overlooks the changes in the latent dimension. To rectify this misconception, we present theoretical evidence in Proposition \ref{prop:1} to substantiate our argument.

\begin{proposition} \label{prop:1}
Given an AE (Fig.~\ref{fig:overview_rec}), let $\mathbf{Z}_0 \in \mathbb{R}^D$ be the feature vector before the latent vector $\mathbf{Z} \in \mathbb{R}^d$, $\hat{\mathbf{Z}}_0 \in \mathbb{R}^D$ be the feature vector after $\mathbf{Z}$. Then, if $d < \frac{D}{2}$, the AE cannot learn identical mapping. 
\end{proposition}
%
%
\begin{proof}
According to \textit{data processing inequality} \cite{cover1999elements}, post-processing cannot increase information. It suggests that, to achieve identical mapping, the bottleneck of AE must preserve all information content for the input. Otherwise, the decoder would be unable to reconstruct the lost information. Therefore, we simplify the analysis by studying identical mapping in the bottleneck.

As shown in Fig.~\ref{fig:overview_rec}, we have $\mathbf{Z} = g_1(\mathbf{Z}_0)$ and $\hat{\mathbf{Z}}_0 = g_2(\mathbf{Z})$, where $g_1$ and $g_2$ are FC layers in the bottleneck. 
Assume $\hat{\mathbf{Z}}_0 \equiv \mathbf{Z}_0$. Considering $\hat{\mathbf{Z}}_0=g_2(g_1(\mathbf{Z}_0))= (\mathbf{Z}_0W^{(1)} + b^{(1)})W^{(2)}+b^{(2)} = \mathbf{Z}_0W^{(1)}W^{(2)} + (b^{(1)}W^{(2)}+b^{(2)})$, we have:
\begin{align}  \label{eq:weight}
    & W^{(1)}W^{(2)} = \mathbf{I}_{D \times D}, \\
    & b^{(1)}= \mathbf{0}_{1 \times d}, b^{(2)} = \mathbf{0}_{1 \times D}.  \label{eq:bias}
\end{align}
Here $W^{(1)} \in \mathbb{R}^{D\times d}, W^{(2)} \in \mathbb{R}^{d\times D}, d^{(1)} \in \mathbb{R}^{d}, d^{(2)} \in \mathbb{R}^{D}$ are learnable weights of $g_1$ and $g_2$. The solution of Eq.~\ref{eq:bias} definitely exits. If Eq.~\ref{eq:weight} has at least one solution, the number of independent scalar equations ($D^2$) should not exceed the number of variables ($2Dd$), i.e., $D^2 \leq 2Dd \Leftrightarrow d \geq \frac{D}{2}$.
As a contraposition, if $d < \frac{D}{2}$, Eq.~\ref{eq:weight} has no solution. Thus, if $d < \frac{D}{2}$, the AE cannot learn identical mapping. \hfill $\square$
\end{proof}

Proposition~\ref{prop:1} reveals that, AE with appropriate latent dimensions can effectively circumvent the undesired identical shortcut. Consequently, we contend that there is no need to introduce more complex modules to address this problem.

\subsection{The optimal solution for AE}
Although AE with $d<\frac{D}{2}$ does not suffer from identical shortcut, we observe that some undesirable abnormal regions are reconstructed due to the generalization ability of the network. This motivates us to theoretically analyze its reason and ideally, find the optimal solution for guiding the design process. Our argument with theoretical evidence is presented in Proposition~\ref{prop:2}.

\begin{figure}[t]
\centering
\includegraphics[width=0.7\linewidth]{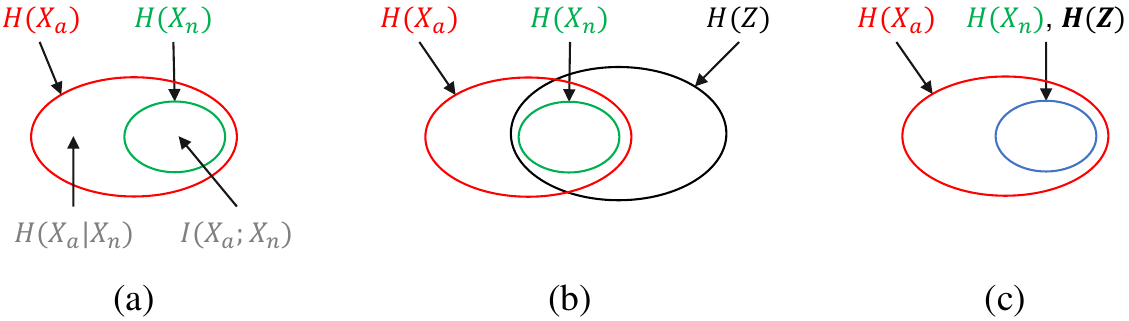}
\caption{Venn diagram of $H(\mathbf{X}_n), H(\mathbf{X}_a), H(\mathbf{Z}$).\textsuperscript{\ref{1}} (a) Relationship between $H(\mathbf{X}_n)$ and $H(\mathbf{X}_a)$; (b) $H(\mathbf{Z})$ of an AE trained with Eq.~\ref{eq:train}; (c) $H(\mathbf{Z})$ of an optimal AE.}
\label{fig:venn}
\end{figure}

\begin{proposition}  \label{prop:2}
Given an AE for anomaly detection, let $\mathbf{X}_n$ be the normal image, $\mathbf{X}_a$ be the abnormal image, and $\mathbf{Z}$ be the latent vector of the AE. An optimal AE should satisfy: (1) $I(\mathbf{X}_n; \mathbf{Z})=H(\mathbf{X}_n)$, and (2) $I(\mathbf{X}_a; \mathbf{Z})=H(\mathbf{X}_n)$.\footnote[1]{$I$ indicates the mutual information and $H$ indicates the information entropy \cite{shannon1948mathematical}.\label{1}}
\end{proposition}

\begin{proof}
Medical abnormal images typically consist of the healthy version with an addition of lesion regions, as shown in Fig.~\ref{fig:overview_rec} ($\mathbf{X}_a$). Therefore, all possible normal patterns could appear in abnormal images. This can be represented as 
\begin{equation} \label{eq:pre}
I(\mathbf{X}_n; \mathbf{X}_a) = H(\mathbf{X}_n)
\end{equation}
As shown in Fig.~\ref{fig:venn}(a), this reveals that the information content of abnormal data $H(\mathbf{X}_a)$ comprises the information content of normal data $H(\mathbf{X}_n)$ and the information content specific to lesions $H(\mathbf{X}_a|\mathbf{X}_n)$.

The reconstruction process can be formulated as a Markov chain: $\mathbf{X} \rightarrow \mathbf{Z} \rightarrow \hat{\mathbf{X}}$. According to \textit{data processing inequality} \cite{cover1999elements}, we have: 
\begin{equation} \label{eq:data_prc}
    I(\mathbf{X};\mathbf{Z}) \geq I(\mathbf{X};\hat{\mathbf{X}}).
\end{equation}

\noindent
(1) For a normal input $\mathbf{X}_n$, according to Eq.~\ref{eq:obj1}, the optimal AE should satisfy $\hat{\mathbf{X}}_n = \mathbf{X}_n$. Combining Eq.~\ref{eq:data_prc}, we have $I(\mathbf{X}_n;\mathbf{Z}) \geq I(\mathbf{X}_n;\hat{\mathbf{X}}_n) = I(\mathbf{X}_n; \mathbf{X}_n) = H(\mathbf{X}_n)$. Using the property of mutual information and joint entropy \cite{cover1999elements}, we have $I(\mathbf{X}_n;\mathbf{Z})=H(\mathbf{X}_n)+H(\mathbf{Z})-H(\mathbf{X}_n,\mathbf{Z}) \leq H(\mathbf{X}_n)$. Hence, $I(\mathbf{X}_n;\mathbf{Z}) = H(\mathbf{X}_n)$. 

\smallskip
\noindent
(2) For an abnormal input $\mathbf{X}_a=\mathbf{X}_n + \delta$, according to Eq.~\ref{eq:obj2}, the optimal AE should satisfy $\hat{\mathbf{X}}_a = \mathbf{X}_n$.
To prevent anomaly $\delta$ from being reconstructed, $\mathbf{Z}$ should not contain abnormal information $H(\mathbf{X}_a|\mathbf{X}_n)$, so $I(\mathbf{Z};\mathbf{X}_a|\mathbf{X}_n)=0$. Combining the property of mutual information \cite{cover1999elements}, we have $I(\mathbf{Z};\mathbf{X}_a) = I(\mathbf{Z};\mathbf{X}_a|\mathbf{X}_n) + I(\mathbf{Z};\mathbf{X}_n;\mathbf{X}_a)=I(\mathbf{Z};\mathbf{X}_n;\mathbf{X}_a)$. Considering $I(\mathbf{Z};\mathbf{X}_n)=H(\mathbf{X}_n)$ and $I(\mathbf{X}_a;\mathbf{X}_n)=H(\mathbf{X}_n)$, we simplify that $I(\mathbf{Z};\mathbf{X}_a)=I(\mathbf{Z};\mathbf{X}_n;\mathbf{X}_a)=I(\mathbf{X}_n;\mathbf{X}_a)=H(\mathbf{X}_n)$.   \hfill $\square$

\end{proof}

Proposition~\ref{prop:2} demonstrates the conditions that the latent vector of an optimal AE should satisfy: (1) it should provide all information content of normal data, but (2) should not contain any information content of abnormal information, as depicted in Fig.~\ref{fig:venn}(c). Previous AE trained with only Eq.~\ref{eq:train} is optimized to achieve $\hat{\mathbf{X}}_n = \mathbf{X}_n$, then it has $ I(\mathbf{X}_n; \mathbf{Z}) = H(\mathbf{X}_n), I(\mathbf{X}_a; \mathbf{Z}) \geq H(\mathbf{X}_n).$
Therefore, this AE satisfies Proposition~\ref{prop:2}(1), but fails to fulfill (2). The Venn diagram in Fig.~\ref{fig:venn}(b) depicts this situation, indicating that its $H(\mathbf{Z})$ extends beyond the scope of $H(\mathbf{X}_n)$ and inadvertently provides information about lesions $H(\mathbf{X}_a|\mathbf{X}_n)$, leading to false negatives. This is an intractable issue since abnormal images are unavailable during training. To address this problem and achieve Proposition~\ref{prop:2}(2), it is ideal for $H(\mathbf{Z})$ to be minimized and approach the same level as $H(\mathbf{X}_n)$, transforming the scope of $H(\mathbf{Z})$ from Fig.~\ref{fig:venn}(b) to (c). This requirement can be formulated as a regularization:
\begin{equation}  \label{eq:min_z}
    \min H(\mathbf{Z}), \text{if}~H(\mathbf{Z}) > H(\mathbf{X}_n).
\end{equation}
Combining Eq.~\ref{eq:train} with \ref{eq:min_z}, $H(\mathbf{Z})$ no longer contains abnormal information but retains normal information, thereby guaranteeing the reconstruction to be normal. 

In summary, our theory suggests that in AD, AE tends to benefit from minimizing the entropy of the latent space, aiming to approach the entropy of normal data. This ensures that anomalies cannot be represented and reconstructed by the model. Meanwhile, for more complex datasets with a higher information content, it is required to increase the entropy of the latent space to match that of the normal data. In practice, this can be achieved explicitly through latent dimension adjustment, or implicitly by enforcing latent space restrictions \cite{kingma2013auto,gong2019memorizing}.

\section{Experiments}  \label{sec:exp}

\subsection{Datasets and implementation details}
We conduct experiments on four datasets, including RSNA\footnote[1]{\url{https://www.kaggle.com/c/rsna-pneumonia-detection-challenge}\label{data1}}, VinDr-CXR\footnote[2]{\url{https://www.kaggle.com/c/vinbigdata-chest-xray-abnormalities-detection}\label{data2}} \cite{nguyen2022vindr}, Brain Tumor\footnote[3]{\url{https://www.kaggle.com/datasets/masoudnickparvar/brain-tumor-mri-dataset}\label{data3}}, and BraTS2021 \cite{baid2021rsna} for evaluation. We follow \cite{cai2022dual,cai2023dual} to build the first three datasets and implement AE. Additionally, we reorganize BraTS2021 \cite{baid2021rsna} to evaluate pixel-level AD. The image-level performance is measured using the area under the ROC curve (AUC) and average precision (AP), while at the pixel level, it is assessed by pixel-level AP ($\text{AP}_\text{pix}$) and the best possible Dice score ($\lceil$Dice$\rceil$). Details of data and implementation are in the Supplementary. 


\subsection{Results and analysis}
\subsubsection{Validation of Proposition~\ref{prop:1}.} Fig.~\ref{fig:rec_err} presents reconstruction errors w.r.t. the latent dimension on RSNA dataset, which shows trends that align with our theory. Firstly, we observe that when $d$ is small, an increase in $d$ results in a decrease of reconstruction errors. While $d>\frac{D}{2}=512$, an increase in $d$ does not lead to smaller errors. This validates that AE with a small $d$ does not encounter identical shortcut, while for $d>\frac{D}{2}$, the bottleneck becomes saturated. Secondly, even when $d>\frac{D}{2}$, errors of normal training data are smaller than those of normal testing data, which, in turn, are smaller than errors of abnormal testing data. This suggests that even if the bottleneck is saturated, identical mapping does not occur in AE due to the limited capacity of the network. 

\begin{figure}[!t]
\centering
\includegraphics[width=0.45\linewidth]{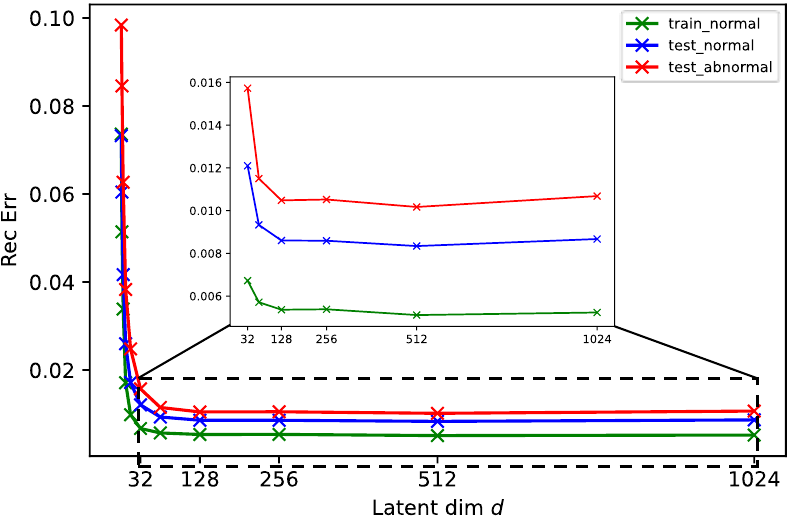}
\caption{Reconstruction errors on RSNA dataset w.r.t. the latent dimension.}
\label{fig:rec_err}
\end{figure}

\begin{table}[!b]
\centering
\caption{Performance of AE with different values of latent dimension $d$.} \label{tab:latent}
\resizebox{\linewidth}{!}{
\begin{tabular}{ccccccccccc}
\toprule
\multirow{2}{*}{\makecell[c]{Latent dim \\ $d$}} & \multicolumn{2}{c}{ RSNA } & \multicolumn{2}{c}{ VinDr-CXR } & \multicolumn{2}{c}{ Brain Tumor } & \multicolumn{4}{c}{ BraTS2021 } \\
\cmidrule(l){2-3}\cmidrule(l){4-5} \cmidrule(l){6-7}  \cmidrule(l){8-11}
                 & AUC & AP & AUC & AP & AUC & AP & AUC & AP & $\text{AP}_\text{pix}$ & $\lceil$Dice$\rceil$ \\ \midrule
128 & $60.5_{\pm 0.6}$ & $60.3_{\pm 0.4}$ & $48.6_{\pm 0.2}$ & $51.9_{\pm 0.3}$ & $86.5_{\pm 0.3}$ & $76.9_{\pm 0.4}$ & $80.5_{\pm 0.9}$ & $90.5_{\pm 0.7}$ & $25.6_{\pm 2.9}$ & $32.7_{\pm 2.3}$     \\
64  & $60.7_{\pm 0.6}$ & $60.7_{\pm 0.6}$ & $49.0_{\pm 0.2}$ & $52.0_{\pm 0.3}$ & $\underline{86.8_{\pm 0.1}}$ & $\underline{77.2_{\pm 0.2}}$ & $81.3_{\pm 0.2}$ & $91.0_{\pm 0.2}$ & $20.8_{\pm 3.1}$ & $28.8_{\pm 2.8}$  \\
32  & $62.8_{\pm 0.4}$ & $62.5_{\pm 0.2}$ & $52.5_{\pm 0.4}$ & $55.6_{\pm 0.5}$ & $\mathbf{87.1_{\pm 0.2}}$ & $\mathbf{77.8_{\pm 0.2}}$ & $\underline{82.1_{\pm 0.3}}$ & $\underline{91.6_{\pm 0.1}}$ & $25.2_{\pm 3.8}$ & $32.7_{\pm 3.3}$  \\
16  & $67.5_{\pm 0.9}$ & $66.7_{\pm 0.5}$ & $56.4_{\pm 0.4}$ & $60.2_{\pm 0.5}$ & $85.9_{\pm 0.3}$ & $76.9_{\pm 0.1}$ & $\mathbf{82.6_{\pm 0.1}}$ & $\mathbf{92.0_{\pm 0.1}}$ & $33.2_{\pm 2.2}$ & $39.2_{\pm 1.7}$  \\
8   & $\underline{69.4_{\pm 1.2}}$ & $\underline{68.3_{\pm 0.8}}$ & $57.4_{\pm 0.3}$ & $61.7_{\pm 0.3}$ & $84.1_{\pm 0.6}$ & $75.5_{\pm 0.8}$ & $81.0_{\pm 0.5}$ & $90.9_{\pm 0.3}$ & $\mathbf{36.1_{\pm 7.9}}$ & $\mathbf{42.3_{\pm 5.9}}$ \\
4   & $\mathbf{72.9_{\pm 2.1}}$ & $\mathbf{70.3_{\pm 0.9}}$ & $\mathbf{61.3_{\pm 0.6}}$ & $\mathbf{64.1_{\pm 0.8}}$ & $70.8_{\pm 1.4}$ & $64.1_{\pm 1.0}$ & $79.4_{\pm 0.1}$ & $89.6_{\pm 0.3}$ & $33.4_{\pm 5.4}$ & $40.0_{\pm 4.2}$ \\
2   & $67.5_{\pm 0.8}$ & $65.4_{\pm 1.0}$ & $\underline{60.5_{\pm 0.3}}$ & $\underline{59.7_{\pm 0.2}}$ & $47.3_{\pm 1.7}$ & $48.3_{\pm 1.0}$ & $77.2_{\pm 0.3}$ & $87.5_{\pm 0.1}$ & $30.0_{\pm 2.5}$ & $38.0_{\pm 1.7}$ \\
1   & $66.5_{\pm 0.4}$ & $63.6_{\pm 0.9}$ & $59.4_{\pm 1.6}$ & $58.6_{\pm 1.6}$ & $36.6_{\pm 0.3}$ & $42.4_{\pm 0.6}$ & $72.8_{\pm 0.8}$ & $84.5_{\pm 0.7}$ & $\underline{35.9_{\pm 5.1}}$ & $\underline{40.8_{\pm 3.4}}$ \\
\bottomrule
\end{tabular}
}
\end{table}

\subsubsection{Validation of Proposition~\ref{prop:2}.} We validate Proposition~\ref{prop:2} by controlling $H(\mathbf{Z})$ via latent dimension adjustment. Tab.~\ref{tab:latent} presents the performance of AE with different $d$ on AD, which aligns with our proposition. Firstly, reducing $d$ from 128 to 1 initially improves the performance and then leads to deterioration, with the optimal $d$ typically being quite small. Notably, the performance of $d=4$ surpasses that of $d=128$ by more than 10\% AUC on both RSNA and VinDr-CXR datasets, indicating that normal information $H(\mathbf{X}_n)$ can be represented by a compact vector. On the other hand, a too large value of $d$ may result in generalization to abnormal samples, while too small value cannot sufficiently represent normal information, leading to performance deterioration. 

Secondly, the optimal $d$ varies across different image modalities, reflecting differences in $H(\mathbf{X}_n)$. For RSNA and VinDr-CXR datasets, the optimal $d$ is 4, whereas for Brain Tumor and BraTS2021 datasets, it is 32 and 16, respectively. This disparity can be attributed to the fact that MRIs offer more information content compared to X-rays. MRIs are volumetric scans that capture detailed tissue information, exhibiting greater variations among healthy subjects and encompassing variations among axial slices. These characteristics enable MRIs to surpass the information content of X-rays, necessitating a larger $d$ to effectively expand $H(\mathbf{Z})$ for MRI datasets.

\subsubsection{Comparison with other methods.} Tab.~\ref{tab:methods} compares the performance of typical AD methods with AE using the optimal latent dimension. Among these methods, MemAE \cite{gong2019memorizing} and VAE \cite{kingma2013auto} incorporate specific designs for latent space restriction, CeAE \cite{zimmerer2018context} utilizes an inpainting task to aid in repairing abnormal regions. Following \cite{cai2022dual}, we set $d=16$ by default for reconstruction methods. 

The results further supports Proposition~\ref{prop:2}. Especially, we observe that directly adjusting the latent dimension in AE (AE[$d_{optimal}$]) yields superior performance compared to more complex approaches that enforce latent space restrictions. AE[$d_{optimal}$] outperforms VAE, MemAE, and CeAE on all four datasets, demonstrating the superiority of the simple and explicit latent dimension adjustment over the more complex implicit restrictions.

\begin{table}[!t]
\centering
\caption{Comparison of typical AD methods. We set $d=16$ by default. $d_{optimal}=4, 4, 32, 16$ for RSNA, VinDr-CXR, Brain Tumor, and BraTS2021 datasets, respectively.}
\label{tab:methods}
\resizebox{\linewidth}{!}{
\begin{tabular}{lcccccccccc}
\toprule 
\multirow{2}{*}{Method} & \multicolumn{2}{c}{ RSNA } & \multicolumn{2}{c}{ VinDr-CXR } & \multicolumn{2}{c}{ Brain Tumor } & \multicolumn{4}{c}{ BraTS2021 } \\
\cmidrule(r){2-3}\cmidrule(r){4-5} \cmidrule(r){6-7} \cmidrule(){8-11}
& AUC & AP & AUC & AP & AUC & AP & AUC & AP & $\text{AP}_\text{pix}$ & $\lceil$Dice$\rceil$ \\
\midrule 
AE            
    & $67.5_{\pm 0.9}$ & $66.7_{\pm 0.5}$ & $56.4_{\pm 0.4}$ & $60.2_{\pm 0.5}$ & $85.9_{\pm 0.3}$ & $76.9_{\pm 0.1}$ & $82.6_{\pm 0.1}$ & $92.0_{\pm 0.1}$ & $33.2_{\pm 2.2}$ & $39.2_{\pm 1.7}$ \\  
VAE \cite{kingma2013auto}        
    & $67.9_{\pm 0.8}$ & $67.0_{\pm 0.8}$ & $56.4_{\pm 0.4}$ & $60.1_{\pm 0.5}$ & $85.3_{\pm 0.4}$ & $76.5_{\pm 0.3}$ & $80.6_{\pm 0.8}$ & $90.9_{\pm 0.5}$ & $\mathbf{44.0}_{\pm 5.3}$ & $\mathbf{47.1}_{\pm 3.4}$ \\
MemAE \cite{gong2019memorizing}  
    & $69.6_{\pm 0.4}$ & $68.1_{\pm 0.4}$ & $56.9_{\pm 0.9}$ & $60.9_{\pm 0.9}$ & $82.6_{\pm 0.6}$ & $74.3_{\pm 0.2}$ & $63.9_{\pm 7.7}$ & $78.5_{\pm 6.9}$ & $22.8_{\pm 5.6}$ & $30.8_{\pm 4.8}$ \\
CeAE \cite{zimmerer2018context}
    & $68.0_{\pm 0.5}$ & $66.8_{\pm 0.6}$ & $56.1_{\pm 0.2}$ & $60.1_{\pm 0.3}$ & $85.3_{\pm 1.5}$ & $76.0_{\pm 1.5}$ & $81.7_{\pm 1.5}$ & $91.5_{\pm 0.7}$ & $28.6_{\pm 3.2}$ & $35.8_{\pm 2.4}$ \\
\rowcolor{green!20}
AE[$d_{optimal}$]
    & $\mathbf{72.9}_{\pm 2.1}$ & $\mathbf{70.3}_{\pm 0.9}$ & $\mathbf{61.3}_{\pm 0.6}$ & $\mathbf{64.1}_{\pm 0.8}$ & $\mathbf{87.1}_{\pm 0.2}$ & $\mathbf{77.8}_{\pm 0.2}$ & $\mathbf{82.6}_{\pm 0.1}$ & $\mathbf{92.0}_{\pm 0.1}$ & $33.2_{\pm 2.2}$ & $39.2_{\pm 1.7}$ \\
\bottomrule
\end{tabular}
}
\end{table}

\section{Conclusion and Discussion}
This paper investigated a theoretical analysis of AE in anomaly detection. We prove that an appropriate latent dimension can avoid "identical shortcut" in AE. By leveraging information theory, the optimal solution of AE is uncovered. Our findings indicate that, apart from the reconstruction loss, imposing a constraint on the entropy of the latent space is crucial for preventing the reconstruction of abnormal information. Experiments validate our theoretical framework, and highlight the efficacy of simple latent dimension reduction in constraining the entropy and achieving significant performance improvements. Overall, this paper provides a theoretical foundation for guiding the design of AE in anomaly detection, facilitating the development of more effective and reliable anomaly detection methods. 

However, the current approach for adjusting the latent dimension relies on evaluation results to find the optimal value, which is undesirable. To overcome this limitation, our future work aims to quantify the information entropy of normal training data, $H(\mathbf{X}_n)$, and develop self-adaptive methods that dynamically constrain $H(\mathbf{Z})$ to approach $H(\mathbf{X}_n)$ on different datasets. This approach would eliminate the need for manual selection of the latent dimension and enhance the adaptability of AE in various anomaly detection scenarios.

\begin{credits}
\subsubsection{\ackname} This work was supported by the Research Grants Council of Hong Kong (No. R6003-22 and T45-401/22-N).

\subsubsection{\discintname}
The authors have no competing interests to declare that are relevant to the content of this article.
\end{credits}
%
%
%
\bibliographystyle{splncs04}
\bibliography{Paper-0616}

\newpage
\section*{\LARGE Supplementary}
\appendix
\renewcommand\thetable{\Alph{section}\arabic{table}}
\renewcommand\thefigure{\Alph{section}\arabic{figure}}

\section{Datasets}
We use four datasets for evaluation, where three are from \cite{cai2023dual}: \\
\textbf{RSNA Dataset}\textsuperscript{\ref{data1}}. The dataset contains 8851 normal and 6012 lung opacity CXRs. In experiments, we use 3851 normal images as the normal training dataset $\mathcal{D}_{train}$, and 1000 normal and 1000 abnormal (lung opacity) images as the test dataset $\mathcal{D}_{test}$. \\
\textbf{VinDr-CXR Dataset}\textsuperscript{\ref{data2}} \cite{nguyen2022vindr}. The dataset contains 10606 normal and 4394 abnormal CXRs that include 14 categories of anomalies in total. In experiments, we use 4000 normal images as $\mathcal{D}_{train}$, and 1000 normal and 1000 abnormal images as $\mathcal{D}_{test}$. \\
\textbf{Brain Tumor Dataset}\textsuperscript{\ref{data3}}. The dataset contains 2000 MRI slices with no tumors, 1621 with glioma, and 1645 with meningioma. The glioma and meningioma are regarded as anomalies. In experiments, we use 1000 normal images (with no tumors) as $\mathcal{D}_{train}$, and 600 normal and 600 abnormal images (300 with glioma and 300 with meningioma) as $\mathcal{D}_{test}$. \\

Additionally, \textbf{BraTS2021} \cite{baid2021rsna} is reorgnized for pixel-level anomaly segmentation in this paper. It provides 1251 MRI cases with the resolution $155 \times 240 \times 240$, as well as the corresponding voxel-level annotation for tumor regions. Each case has multiple modalities, including T1, T1ce, T2, and flair. We use only the flair in experiments as it is the most sensitive to tumor regions. In preprocessing, scans are cropped into $70 \times 208 \times 208$ to remove the empty corners. To construct $\mathcal{D}_{train}$, 4211 normal 2D axial slices are extracted from 1051 MRI scans, while the extracted slices are at least 5 slices apart from each other, ensuring enough different content between neighbors. For $\mathcal{D}_{test}$, similarly, 828 normal and 1948 tumor slices are extracted from the remaining 200 MRI scans. 

Tab.~\ref{tab:data_repartition} summarizes these datasets.

\begin{table}
\centering
\caption{Summary of the datasets.}   \label{tab:data_repartition}

\begin{tabular}{llll}
\toprule
\multicolumn{1}{l}{\multirow{2}{*}{\textbf{Dataset}}}       & \multicolumn{1}{l}{\multirow{2}{*}{\textbf{Modality}}}   &\multicolumn{2}{l}{\textbf{Repartition}}      \\ \cmidrule(){3-4} 
\multicolumn{1}{c}{}                                        & \multicolumn{1}{c}{}           & $\mathcal{D}_{train}$ & $\mathcal{D}_{test}$ (Normal+Abnormal)   \\ \midrule
RSNA            & Chest X-ray                    & 3851                  & 1000+1000     \\
VinDr-CXR       & Chest X-ray                    & 4000                  & 1000+1000     \\
Brain Tumor     & Brain MRI                      & 1000                  & 600+600       \\
BraTS2021       & Brain MRI                      & 4211                  & 828+1948      \\
\bottomrule
\end{tabular}
\end{table}

\section{Implementation details}
The AE in our experiments comprises an encoder and a decoder. The encoder consists of four convolutional layers with kernel size 4 and stride 2, whose channel sizes are 16-32-64-64. The decoder consists of four deconvolutional layers with the same kernel size and stride as the encoder, and the channel sizes are 64-32-16-1. The encoder and decoder are connected by four fully connected layers, where the dimensions are 1024-$D$-$d$-$D$-1024. $D$ is always set to 1024 in our experiments, and $d$ is the latent dimension of AE. All layers except the output layer are followed by batch normalization (BN) and ReLU. For fair comparison, other reconstruction methods are implemented based on this backbone consistently. All the input images are resized to $64 \times 64$, and all the reconstruction models are trained for 250 epochs using the Adam optimizer with a learning rate of 1e-3. 

The performance on image-level anomaly detection is assessed with the area under the ROC curve (AUC) and average precision (AP). The performance on pixel-level anomaly segmentation is assessed with the pixel-level AP ($\text{AP}_\text{pix}$) and the best possible Dice score ($\lceil$Dice$\rceil$).

\end{document}